%% file: clear2025.tex
\documentclass[final,12pt]{clear2025} % Anonymized submission
\usepackage{subcaption}
\usepackage{tikz}
\usepackage{caption}
\usepackage{mathtools}
\usepackage{bbm}
\usepackage{algorithm}
\usepackage{algorithmic}
\usepackage{pgfplots}
\usepackage{multicol}
\input{macros.tex}
\usepackage{xcolor}
\usepackage{siunitx}
\usepackage{xcolor}

\usetikzlibrary{shapes,decorations,arrows,calc,arrows.meta,fit,positioning,shapes.geometric}
\tikzset{
    auto,node distance =2 cm and 2 cm,semithick,
    state/.style ={ellipse, draw, minimum width = 0.3 cm},
    point/.style = {circle, draw, inner sep=0.04cm,fill,node contents={}},
    bidirected/.style={Latex-Latex,dashed},
    el/.style = {inner sep=2pt, align=left, sloped}
}
\tikzset{arrow/.style={-stealth, thick, draw=gray!80!black}}
\title[Transfer learning through causal transportability]{Transfer Learning in Latent Contextual Bandits with Covariate Shift Through Causal Transportability}
\usepackage{times}
 \clearauthor{\Name{Mingwei Deng} \Email{mingwei.deng@aalto.fi}\\
  \Name{Ville Kyrki} \Email{ville.kyrki@aalto.fi}\\
  \Name{Dominik Baumann} \Email{dominik.baumann@aalto.fi}\\
  \addr Department of Electrical Engineering and Automation\\
  Aalto University\\
  Espoo, Finland}

\begin{document}

\maketitle

\begin{abstract}%
Transferring knowledge from one environment to another is an essential ability of intelligent systems. 
Nevertheless, when two environments are different, %effective knowledge transfer is challenging.
naively transferring all knowledge may deteriorate the performance, a phenomenon known as \emph{negative transfer}.  In this paper, we address this issue within the framework of multi-armed bandits from the perspective of causal inference. Specifically, we consider transfer learning in \emph{latent contextual bandits}, where the actual context is hidden, but a potentially high-dimensional proxy is observable. 
We further consider a \emph{covariate shift} in the context across environments. We show that naively transferring all knowledge for classical bandit algorithms in this setting led to negative transfer. 
We then leverage \emph{transportability} theory from causal inference to develop algorithms that explicitly transfer effective knowledge for estimating the causal effects of interest in the target environment. Besides, we utilize variational autoencoders to approximate causal effects under the presence of a high-dimensional proxy. We test our algorithms on synthetic and semi-synthetic datasets, empirically demonstrating consistently improved learning efficiency across different proxies compared to baseline algorithms, showing the effectiveness of our causal framework in transferring knowledge.
\end{abstract}


\section{Introduction}

% Sequential decision-making problems are often formulated as contextual bandit problems when side information is available to the decision-maker. In these problems, an agent seeks to find a policy that maximizes the expected cumulative reward by sequentially selecting an action among a set of available actions. %A common approach to this problem involves balancing the exploration and exploitation of the reward distributions across different action-context pairs, which%
% Typical methods require collecting a large amount of data from the environment, especially when the side information is high-dimensional and serves merely as a proxy of the latent context that directly influences the reward. %In the real-world scenario, gathering such extensive data is often impractical. This challenge is further compounded when the side information is high-dimensional and serves merely as a proxy for a latent variable that directly influences the reward—a situation that frequently arises in practice. Thus, improving sample efficiency is critical.
% For example, assuming we have some available policies for driving, training an autonomous driving model to select from these policies for different weather using only camera observations would require collecting lots of data from the environment \citep{zhang2023perception}. 

Consider an autonomous vehicle equipped with several pre-trained driving policies, each optimized for different weather conditions. 
We want to train an agent to select the driving policy that maximizes the expected cumulative reward based on camera images.
Thus, it must infer the relevant weather conditions from these high-dimensional images. 
Typically, this would require extensive data collection from the environment, which, in real-world experiments with the car, is time-consuming and costly \citep{zhang2023perception}. 
Therefore, data-efficient algorithms are needed.
The car example can be formulated as a latent contextual bandit problem, where the agent only has access to a high-dimensional proxy variable instead of the actual context.
%In this paper, we propose leveraging \emph{transportability} to transfer knowledge, thereby improving sample efficiency in such problem settings.

One approach to improve the sample efficiency is transfer learning, which allows the agent to reuse knowledge gained from a related environment to accelerate the learning process in the target environment. The expectation is that the transferred knowledge can serve as a warm start for the agent to converge to the optimal policy (e.g., the optimal selection of driving policy in the example above) faster. However, transferring knowledge may lead to \emph{negative transfer} \citep{wang2019characterizingavoidingnegativetransfer} if the environment difference is too large. For instance, if we reuse the knowledge of driving a car from a dataset collected in the summer but want to train a model to drive in the winter, the model may initially be too aggressive for the winter conditions. %Further, in this example, another problem arises: humans choose their driving policies based on the low-dimensional context, e.g., weather, rather than the raw observation from the camera. This makes the latent context a confounder in the dataset. Directly transferring the knowledge of driving based on raw observation does not reflect the causal effect of driving policies on driving performance in the dataset.

Unlike earlier works on transfer learning in bandit problems~\footnote{A more detailed review of related work is provided in Appendix \ref{Appd:related work}.} \citep{deshmukh2017multi, zhang2017transfer, liu2018transferable, cai2024transfer, bellot2024transportability}, this paper focuses on tackling the challenge of negative transfer in \emph{latent contextual bandits} \citep{zhou2016latent, sen2017contextual}. %This allows us to avoid transferring the biased causal effect of actions on the reward.
Specifically, the objective is to define a contextual bandit agent that can improve sample efficiency by exploiting a dataset collected in another environment, given the awareness of the potential environmental discrepancy, while avoiding negative transfer.  In causal inference, the study of \emph{transportability} \citep{pearl2011transportability, bareinboim2012transportability,bareinboim2016causal} focuses on determining what knowledge can be transferred from one environment to infer causal effects in another, where the two environments are assumed to be different. Our approach employs the \emph{transport formula} derived from the transportability theory to explicitly determine the knowledge we can transfer from the data. Furthermore, to address the high-dimensional proxy, our approach builds on \emph{Variational Autoencoders (VAEs)}, which are well-suited for capturing the structure of latent variables across various high-dimensional datasets \citep{hu2017toward, lin2020anomaly,peng2021generating} and infer the causal effect of interventions \citep{louizos2017causal}. 

\vspace{1mm}\noindent\textbf{Contributions.} We propose a novel latent contextual bandit learning algorithm that can transfer knowledge across different environments based on the transport formula and improve sample efficiency. We test our method on several synthetic and semi-synthetic datasets,  including settings where images serve as high-dimensional proxies. Our results show that the proposed method successfully avoids negative transfer compared to methods that naively extract knowledge from the data collected in another environment. Also, it consistently achieves higher sample efficiency over methods that do not consider any prior data.   

\section{Problem Formulation and Background}

We consider a transfer learning setting for a latent contextual bandit problem, in which an agent seeks to minimize its cumulative regret in the target environment given a dataset collected from another environment. In this section, we introduce the concrete problem setting and necessary background to build our method. Then, based on the theory of transportability, we formalize the problem.

Throughout the paper, we use capital letters to denote random variables and small letters to denote their realizations, e.g., $X=x$. For sets of random variables, we use bold capital and small letters to denote them and their realizations, e.g., $\boldsymbol{X}=\boldsymbol{x}$. A distribution of a random variable $Y$ is denoted by $\P_Y$, and its conditional distribution is denoted by $\P_{Y\vert x}$, an abbreviation of $\P_{Y\vert X=x}$. The probability of a specific value $x$ taken by $X$ is denoted by $p(x)$, an abbreviation of $p(X=x)$.
\subsection{Problem Setting}
\label{problem formulation}

We start by considering a contextual bandit problem in which the true context variable is latent, but we have access to its proxy. Specifically, let the context variable be $Z \in \mathcal{Z} \subseteq \mathbb{R}^{d_1}$ and the proxy variable $W \in \mathcal{W} \subseteq \mathbb{R}^{d_2}$, where $d_1 \ll d_2$.  At time step $t$, the environment samples $z_t$ from a marginal distribution $\P_Z$ and $w_t$ from a conditional distribution $\P_{W\vert z_t}$, but only $w_t$ is revealed to the agent. From the observed $w_t$, the agent is required to select an arm $x_t \in \mathcal{X} = \{1,\ldots, K\}$% according to some policy $\mu: \mathcal{O} \rightarrow \mathcal{X}$%
. In return, it receives a reward $y_t \in \mathcal{Y} \subseteq \mathbb{R}$. Different from the standard contextual bandit setting \citep{langford2007epoch}, where the reward is sampled from a conditional distribution $P_{Y\vert w_t, x_t}$, the reward $y_t$ here is drawn from $\P_{Y\vert z_t,x_t}$. Let the optimal reward at step $t$ be $y^{\ast}_t$. %$u^{\ast}_t$%.
 The objective of the agent is to minimize the expected cumulative regret
\begin{equation}
  \mathbb{E}[R_T] = \sum_{t=0}^{T}\mathbb{E}[y^{\ast}_t] - \mathbb{E}[y_t],
\end{equation}
where $T$ is the total number of time steps. % A classic way to solve such a problem is to estimate the reward distribution $P_{Y\vert Z, X}$ and then choose the optimal arm $\tilde{x}=\mathrm{\argmax}_{x \in \mathcal{X}}\mathbb{E}(Y\vert z, x)$. However, $Z$ is assumed unobserved in this problem. One may instead estimate the reward distribution given the proxy, i.e., $P_{Y\vert W,X}$ and choose the optimal arm $\tilde{x}=\mathrm{\argmax}_{x \in \mathcal{X}}\mathbb{E}(Y\vert w, x)$. Nevertheless, given the high-dimensional nature of the proxy, this method is not sample-efficient.  

We then extend the problem to a transfer learning setting, where we have two domains: the source domain $\pi$ and the target domain $\pi^{\ast}$. The agent is given a fixed prior dataset collected from  $\pi$. It includes $W, Z, X, Y$, and the arm $X$ is selected based on $Z$.
%It contains $N$ rounds of interactions of another agent that can observe the context variable $Z$ and its proxy $W$ at the same time, and its selection of the arm $X$ is based on the context variable $Z$. %
Formally, we denote the dataset to be $D := \{(w_i, z_i, x_i, y_i)\}^{N}_{i=1}$.  Different from the \textit{offline contextual bandit} setting \citep{strehl2010}, where $\pi^{\ast}$ is assumed identical to $\pi$, we assume a \textit{covariate shift} \citep{sugiyama2007covariate} on the latent context variable $Z$ between the domains. 
\begin{assumption} [Covariate shift on bandit with latent context]
  \label{covariate shift}
  Given two domains~$\pi$ and~$\pi^{\ast}$, the marginal distribution of the latent context variable is different in the two domains, i.e., $\P^{\pi}_Z \neq \P^{\pi^{\ast}}_{Z}$, while the conditional distributions that generate the proxy~$W$ and the reward~$Y$ remain the same, i.e., $\P^{\pi}_{W\vert Z}=\P^{\pi^{\ast}}_{W\vert Z}$ and $P^{\pi}_{Y\vert Z, X}=P^{\pi^{\ast}}_{Y\vert Z, X}$.    
\end{assumption}

Under this assumption, the objective of the agent is to minimize the cumulative regret in the target domain more efficiently by leveraging the prior dataset $D$.

\subsection{Background}
Our proposed method builds upon the framework of \textit{structural causal models} and transportability theory, which we introduce in the following.

\subsubsection{Structural Causal Models}
A structural causal model (SCM) $M$ is a tuple $M=\langle\boldsymbol{V},\boldsymbol{U}, \mathcal{F}, \P_{\boldsymbol{U}} \rangle$, where $\boldsymbol{V}$ is a set of endogenous variables that can be observed, $\boldsymbol{U}$ is a set of exogenous variables that are unobservable, $\mathcal{F}=\{f_1,\ldots,f_{|\boldsymbol{V}|}\}$ is a set of functions that represent the causal mechanisms generating $X_i$ from its parents $\boldsymbol{Pa}_i \subseteq (\boldsymbol{V} \cup \boldsymbol{U}) \setminus \{{X_i} \}$,  i.e., $X_i \gets f_i(\boldsymbol{Pa}_i)$, and $\P_{\boldsymbol{U}}$ is a joint distribution over the exogenous variables $\boldsymbol{U}$.

An SCM entails both an observational distribution and an \textit{interventional distribution} \citep{Elements}.  An intervention is defined as an operation that fixes values of a subset $\boldsymbol{X}\subset \boldsymbol{V}$ to a constant $\boldsymbol{x}$, denoted by $do(\boldsymbol{x})$, an abbreviation of $do(\boldsymbol{X}=\boldsymbol{x})$, which replaces functions $\{f_i: X_i \in \boldsymbol{X}\}$ %not a proper expression%
that determine values of $\boldsymbol{X}$ in the original SCM $M$ and generates a post-intervention SCM $M_{\boldsymbol{x}}$. The entailed distribution over $\boldsymbol{V}$ in $M_x$ is the interventional distribution induced by $do(\boldsymbol{x})$, denoted by $P_{\boldsymbol{V}_{\boldsymbol{x}}}\coloneqq P_{\boldsymbol{V} \vert do(\boldsymbol{x})}$. 

Every SCM $M$ has an associated \textit{causal graph} $\mathcal{G}=(\boldsymbol{V}, \mathcal{E})$, which is a \textit{directed acyclic graph} (DAG). Let $\boldsymbol{V_i} = \boldsymbol{Pa}_i \subseteq \boldsymbol{V}$ and $\boldsymbol{U_i} = \boldsymbol{Pa}_i \subseteq \boldsymbol{U}$. Given an SCM $M$, if $X_i \subseteq \boldsymbol{V}_j$, %how to make sure i is not equal to j%
 we draw a directed edge from $X_i$ to $X_j$ ($X_i \rightarrow X_j$) in a corresponding causal graph $\mathcal{G}$. If $X_i \subseteq \boldsymbol{U}_j$, we draw a dashed edge from $X_i$ to $X_j$. Fig.~\ref{fig:sourceG} and Fig.~\ref{fig:targetG} show corresponding causal graphs for the data-generating process of the latent contextual bandit in the source domain and the target domain, respectively. The causal graph of a post-intervention SCM $M_{\boldsymbol{x}}$ induced by $do(\boldsymbol{x})$ is similar to the original SCM $M$, except that the incoming edges of nodes associated with $\boldsymbol{X}$ are removed.

 \subsubsection{Transportability}
%We are considering a covariate shift between the source domain $\pi$ and the target domain $\pi^{\ast}$, and it is important to identify what knowledge we should transfer from the source domain to minimize the regret in the target domain efficiently.  
In the causal inference literature, the problem of identifying causal effects with potential domain discrepancy is studied under the theory of transportability. These works focus on identifying whether a causal effect can be estimated across domains and what can be transferred. Formally, the transportability is defined
\begin{definition}[Transportability] \citep{bareinboim2012transportability}
  \label{transportability}
  Let $\pi$ and $\pi^{*}$ denote two domains, characterized by two observational distributions $\P_{\boldsymbol{V}}$, $\PS_{\boldsymbol{V}}$, and two causal graphs $\mathcal{G}$ and $\mathcal{G}^{\ast}$, respectively. A causal effect $p(y \vert  do(x))$ is transportable from $\pi$ to $\pi^{\ast}$ if $p(y \vert  do(x))$ can be estimated from a set of interventions $\mathcal{I}$ on $\pi$, and $p^{\ast}(y \vert  do(x))$ can be identified from $\P_{\boldsymbol{V}}$, $\PS_{\boldsymbol{V}}$, $\mathcal{G}$, $\mathcal{G}^{\ast}$ and $\mathcal{I}$.
\end{definition}
Note, in Definition~\ref{transportability}, $p(y \vert  do(x))$ is identified from the source domain $\pi$, which is not necessarily the same as $p^{\ast}(y \vert  do(x))$ in the target domain $\pi^{\ast}$. If it is, we say $p^{\ast}(y\vert do(x))$ is \text{directly transportable}. To check the transportability of a causal effect, the first step is to encode the domain discrepancy into a \textit{selection diagram} $\mathcal{D}$ by adding a \textit{selection node} pointing at the node that has a potential structural discrepancy. Fig.~\ref{fig:SD} shows a selection diagram encoding the domain discrepancy under Assumption~\ref{covariate shift}. Given the selection diagram between two domains, one can identify whether a causal effect is transportable by applying the $do$-calculus \citep{pearl2009causality}  on the causal effect and the corresponding selection diagram. If the causal effect can be transformed into an expression that only includes available interventional distributions and observational distributions in the source domain and the observational distribution in the target domain, we say the causal effect is transportable. The resulting expression of the causal effect is called a transport formula. 

%\begin{definition}[Selection Diagram]
%Let $\langle M, M^{\ast} \rangle$ be a pair of SCMs associated to domains $\langle \pi, \pi^{\ast} \rangle$ that share the same causal graph $\mathcal{G}$. A selection diagram $\mathcal{D}$ can be constructed as follows:
%\begin{itemize}
%  \item Every edge in $\mathcal{G}$ is an edge in $\mathcal{D}$;
%  \item Selection diagram $\mathcal{D}$ contains an extra edge $S_i \rightarrow X_i$ whenever there might exist a discrepancy $f_i \neq f_i^\ast$ or $\P_{U_i}\neq \PS_{U_i}$ between $M$ and $M^{\ast}$ $ \forall X_i \in \boldsymbol{V}$.
%\end{itemize}
%\end{definition}

%Here, $S_i$ is called the selection node, which indicates that there might be a structural discrepancy between the two domains. The absence of a selection node pointing to $X_j \in \boldsymbol{V}$ represents the assumption that the causal mechanism responsible for assigning a value to $X_j$ is identical between two domains.

\begin{figure}[t]
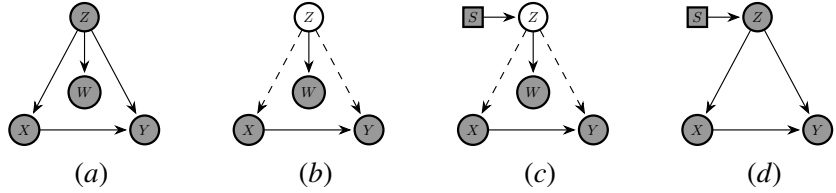

  \floatconts
  {fig:example2}% label for whole figure
  {\caption{Causal Graphs of the data-generating processes. \capt{Gray nodes are observable and white nodes are unobservable. Graphs in ($a$) and ($b$) depict causal graphs in the source domain and the target domain, respectively. The graph in ($c$) is the corresponding selection diagram for the two domains: the selection node $S$ points to $Z$ since $\P_{Z}\neq \PS_{Z}$. The graph in ($d$) is a selection diagram encoding the domain discrepancy of an offline contextual bandit under a covariate shift.}}}% caption for whole figure
  {%
  \subfigure{%
  \label{fig:sourceG}% label for this sub-figure
    \includeteximage{tikz/source.tikz}
  }\qquad % space out the images a bit
  \subfigure{%
  \label{fig:targetG}% label for this sub-figure
    \includeteximage{tikz/target.tikz}
    }\qquad
    \subfigure{%
    \label{fig:SD}% label for this sub-figure
      \includeteximage{tikz/SD.tikz}
      }\qquad
    \subfigure{%
  \label{fig:SD_}% label for this sub-figure
    \includeteximage{tikz/SD_.tikz}
    }
  }
  \end{figure}

\subsection{Transportability of Bandits with Latent Contexts}
\label{sec:2.3}
Building on the preliminary discussion of transportability, we can formulate the problem described in Section~$\ref{problem formulation}$ within this framework. Recall that the agent’s goal is to minimize the cumulative regret in the target domain while leveraging access to a dataset collected from the source domain. In this context, the causal effects of an arm selection conditioned on observing the proxy $w$ are $p(y \vert w, do(x))$ and $p^{\ast}(y \vert  w, do(x))$ in the source and target domains. We refer to this causal effect as a $w$-specific causal effect throughout the paper. Following the definition by \cite{bellot2024transportability}, the problem outlined in Section $\ref{problem formulation}$ can be formulated formally as follows: let $M$ and $M^{\ast}$ denote the SCMs that characterize the source domain $\pi$ and the target domain $\pi^{\ast}$, and $P_{W, Z, X, Y}$ and $\PS_{W, Z, X, Y}$ the distributions entailed by $M$ and $M^{\ast}$. Given the selection diagram $\mathcal{D}$ in Fig.~$\ref{fig:SD}$ that characterizes the domain discrepancy between $\pi$ and $\pi^{\ast}$ and a prior dataset $D \sim \P_{W,Z,X,Y}$, at each step $t=1,\ldots, T$ in the target domain, the agent observes a proxy $w_t \in W $ generated from the underlying context $z_t \in Z$. It then takes an action $x_t \in X$ and observes a reward $y_t \in Y$. The objective of the agent is minimizing the expected cumulative regret in $\pi^{\ast}$, i.e., minimizing
  \begin{equation}
    \label{eq:expected regret}
    \mathbb{E}_{\pi^{\ast}}[R]\coloneqq \sum_{t=1}^{T}\mathbb{E}_{\pi^{\ast}}[Y\vert w_t, do(\tilde{x}_t)] - \mathbb{E}_{\pi^{\ast}}[Y\vert w_t, do(x_t)],
  \end{equation} 
  where $\tilde{x}_t=\mathrm{\argmax}_{x \in \mathcal{X}}\mathbb{E}_{\pi^{\ast}}[Y\vert w_t, do(x)]$. When the proxy variable $W$ is high-dimensional, calculating $\mathbb{E}_{\pi^{\ast}}[Y\vert w_t, do(\tilde{x}_t)]$ analytically is infeasible. We therefore replace it with $\mathbb{E}_{\pi^{\ast}}[Y\vert z_t, do(\tilde{x}_t)]$ in~\eqref{eq:expected regret} to denote the optimal reward in this case. Nevertheless, this results in an ill-posed regret definition, since $\mathbb{E}_{\pi^{\ast}}[Y\vert z_t, do(\tilde{x}_t)]$ always yields a higher return than $\mathbb{E}_{\pi^{\ast}}[Y\vert w_t, do(x_t)]$.
  As the agent cannot close this gap between $\mathbb{E}_{\pi^{\ast}}[Y\vert z_t, do(\tilde{x}_t)]$ and $\mathbb{E}_{\pi^{\ast}}[Y\vert w_t, do(x_t)]$, the best it can do is achieving a linear regret. %In this case, 
  Therefore, with the high-dimensional proxy, we evaluate the value of $\mathbb{E}_{\pi^{\ast}}[R]$ after $T$ steps as the total cumulative regret, rather than evaluating $\mathbb{E}_{\pi^{\ast}}[R]$ per step, which is the conventional evaluation metric for bandit algorithms.
 
From~\eqref{eq:expected regret}, we can see that with an accurate estimation of the casual effect~$\mathbb{E}_{\pi^{\ast}}[Y\vert w_t, do(x)]$ or~$p^{\ast}(y\vert w_t, do(x))$, the expected cumulative regret can be minimized. Hence, our main focus in this paper is estimating this quantity using the prior data and the observational distribution in the target domain.

\section{Estimating Causal Effects with Transportability}
In this section, we derive two algorithms to estimate the causal effect  $p^{\ast}(y\vert w, do(x))$  using prior data while avoiding negative transfer. One algorithm addresses a simple binary model, and the other generalizes the method to a high-dimensional proxy variable. We begin by highlighting the challenges of latent contexts and demonstrating negative transfer in common bandit algorithms. 
%For the binary model, we propose Alg.~\ref{alg:binary}, which transfers knowledge from the source domain while learning from target domain observations. We then extend this approach to a high-dimensional proxy setting using VAEs, leading to Alg.~\ref{alg:general}.

%In this section, we derive two algorithms to estimate the causal effect $p^{\ast}(y\vert w, do(x))$ leveraging the prior data and avoid the negative transfer. One for a simple binary model and another for a more generalized proxy variable, where the proxy is assumed to be high dimensional. We start by illustrating the challenge of latent contexts and showing the negative transfer for common bandit algorithms under the setting we described in Section~\ref{sec:2.3}. Then, we consider a binary model and propose Algorihtm~\ref{alg:binary} to transfer knowledge from the source domain and learn from the observation in the target domain. Lastly, we generalize the method we used in the binary model to a high-dimensional proxy variable setting. By combining with VAEs that can capture the latent structure of the data-generating process in high-dimensional data, we propose Algorithm~\ref{alg:general} to transfer knowledge from the prior data and learn from the observations in the target domain.
\subsection{Challenge of Negative Transfer}
\label{sec:challenge on unobserved context}
We first consider an offline contextual bandit setting, where the context is observable in both domains. Based on this assumption, one can draw a selection diagram of this setting under covariate shift, which is shown in Fig.~\ref{fig:SD_}. Since the agent can observe the context $Z$ in both domains, it can evaluate $\tilde{x}=\mathrm{\argmax}_{x\in \mathcal{X}}\mathbb{E}_{\pi^{\ast}}[Y \vert z, do(x)]$ from the dataset $D$ collected in the source domain. Therefore, the problem is equivalent to identifying the causal effect $\mathbb{E}_{\pi^{\ast}}[Y \vert z,do(x)]$ from the observational distribution $\P_{X,Y,Z}$. Here, the causal effect of interest is directly transportable, and the optimal policy learned in the source domain can thus be directly transferred to the target domain.
\begin{proposition}
  \label{prop1}
  For an offline contextual bandit under covariate shift, the $z$-specific causal effect $p(y\vert z, do(x))$, where $z\in \mathcal{Z}$, $x \in \mathcal{X}$ and $y \in \mathbb{R}$, is directly transportable. The transport formula of $p^{\ast}(y\vert z, do(x))$ is
  \begin{equation}
    \label{eq:TF1}
    p^{\ast}(y\vert z, do(x))=p(y\vert z, x).
  \end{equation}
\end{proposition}

The proof can be found in Appendix \ref{appd:props}.
This proposition suggests that the transportability of the $z$-specific causal effect is trivial and formally justifies the use of supervised learning in offline contextual bandits under covariate shift. %need citations here%. 

One may consider substiting $Z$ with its proxy $W$ to derive the same result as in Proposition~\ref{prop1}, \ie~$p^{\ast}(y\vert w, do(x))=p(y\vert w, x)$. Therefore, concluding that the problem outlined in Section.~\ref{sec:2.3} is also straightforward to solve, as estimating $p(y\vert w,x)$ from the prior dataset $D$ is trivial. However, this conclusion is not valid.

\begin{proposition}
  \label{prop2}
  For the problem described in Section.~\ref{sec:2.3}, the $w$-specific causal effect $p^{\ast}(y\vert w,do(x))$ is not directly transportable and the transport formula of $p^{\ast}(y\vert w,do(x))$ is
  \begin{equation}
    \label{eq:TF2}
    p^{\ast}(y\vert w,do(x))=\sum_{z\in\mathcal{Z}}p(y\vert z,x)p^{\ast}(z\vert w).
    %the summation should probabaly change to a intergal as we assume Z is continuous%
  \end{equation} 
\end{proposition}
The proof can be found in Appendix~\ref{appd:props}.
Proposition \ref{prop2} formally establishes that the $w$-specific causal effect is not directly transportable from the source domain. If one naively transfers $p(y \vert w, do(x))$ from the source domain as if it were $p^{\ast}(y \vert w, do(x))$, it can lead to a negative transfer. This is due to the fact that $p(z\vert w) \neq p^{\ast}(z \vert w)$ under the assumption of covariate shift on $Z$, \ie $p(z) \neq p^{\ast}(z)$. We simulate two experiments on the binary model and linear Gaussian model to show the negative transfer in Fig.~\ref{fig:negative transfer}. The results demonstrate a clear negative transfer for both models, where agents that naively transfer $p(y \vert w, do(x))$ (CTS$^{-}$ and LinUCB$^{-}$) from the source domain suffer a linear increase of regret while regrets for agents that directly learn from the target domain (CTS and LinUCB) increase sub-linearly. The experiment details can be found in Appendix~\ref{Appd: negative transfer}.
\begin{figure}
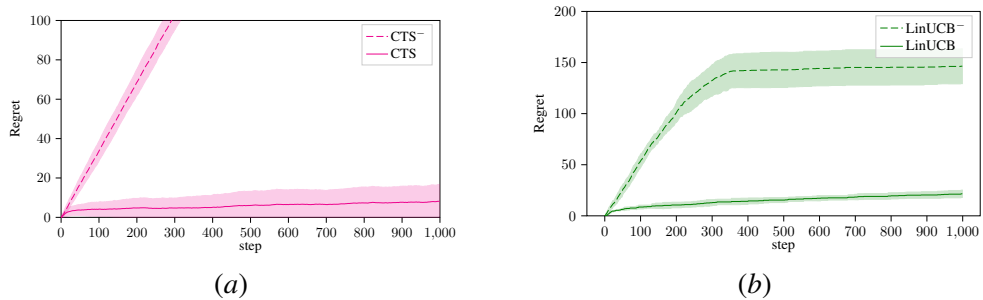

  \centering
  \subfigure{
    \label{fig:negative cucb}
  \includeteximage[width=0.4\textwidth]{tikz/challenge_binary.tex}}\qquad
  \subfigure{
    \label{fig:negative linucb}
  \includeteximage[width=0.4\textwidth]{tikz/challenge_linear.tex}}
  \caption{{Negative transfer for classic bandit algorithms. \capt{Fig.~($a$) shows negative transfer for binary model. Agents in this experiment utilize Thompson sampling \citep{thompson1933likelihood}. Fig.~($b$) shows the negative transfer for the Linear Gaussian model with a scalar latent context $Z$ and 5-dimensional proxy $W$. Agents in this experiment utilize LinUCB \citep{li2010contextual}.}}}
  \label{fig:negative transfer}
\end{figure}

Moreover, although~\eqref{eq:TF2} includes the quantity $p(y\vert z,x)$ that can be estimated from the prior dataset $D$, it is still not possible to fully compute~\eqref{eq:TF2}, as the posterior distribution $p^{\ast}(z\vert w)$ is not estimable without observing the context $Z$ in the target domain. Consequently, this raises concerns about the formulation of the problem in Section~\ref{sec:2.3} within the framework of transportability. However, in the next section, we will show that it is possible to recover the unknown posterior distribution from the observational distribution in the target domain. This allows us to leverage~\eqref{eq:TF2} to estimate the $w$-specific causal effect in the target domain using prior data.  

\subsection{Transfer Learning Guided by Transport Formula}
This section describes a transfer learning algorithm guided by the transport formula we derived in Proposition~\ref{eq:TF2}. We start by restoring the unknown posterior distribution $p^{\ast}(z\vert w)$.

\subsubsection{Restoring Unknown Posterior for Binary Model}

We first consider a binary model~\footnote{We notice the connection between the binary model and a transfer learning problem of \emph{bandits with unobserved confounders}. Interested readers are referred to Remark~\ref{connection to MAB with confounders} in Appendix~\ref{Appd: negative transfer}}. In this case, one can derive an analytic solution for computing the unknown posterior with Bayes' theorem
\begin{equation}
  \label{eq:bayes}
  p^{\ast}(z \vert w) = \frac{p^{\ast}(w\vert z)p^{\ast}(z)}{p^{\ast}(w)}.
\end{equation}
If we can estimate the three terms on the right-hand side, we can solve for the unknown $p^{\ast}(z\vert w)$. Under the assumption of covariate shift, we have $p^{\ast}(w\vert z)=p(w \vert z)$, where $p(w \vert z)$ can be estimated from the prior dataset $D$. The denominator, $p^{\ast}(w)$, can be directly computed from the observation of $W$, given that $W$ is binary. Therefore, the only term that can not be directly computed is $p^{\ast}(z)$. For the binary model, we have the following proposition
\begin{proposition}
 Let $Z$ and $W$ in Fig.~\ref{fig:SD} be binary. Then, the distribution of $Z$ in the target domain can be estimated by  
\begin{equation}
  \label{eq:p(z)}
  p^{\ast}(z=1) = \frac{p^{\ast}(w=1)-p(w=1\vert z=0)}{p(w=1\vert z=1)-p(w=1\vert z=0)}.
\end{equation}
\end{proposition}   
The derivation of~\eqref{eq:p(z)} is based on the marginalization of $p^{\ast}(w=1)$; we provide the details in Appendix~\ref{Appd: Derivation for 6}. Note, the terms on the right-hand side of~\eqref{eq:p(z)} can be estimated either from the observation in the target domain or the prior dataset collected from the source domain.

Eq.~\eqref{eq:p(z)} provides a complete solution to~\eqref{eq:bayes}, thus restoring the unknown posterior $p^{\ast}(z\vert w)$ from the prior dataset and the observation in the target domain. Once the posterior is restored, we can calculate the $w$-specific causal effect by the transport formula derived in~\eqref{eq:TF2}. The concrete algorithm is given in Alg.~\ref{alg:binary}. However, when we drop the assumption of the binary model and consider a high-dimensional proxy $W$ governed by a flexible generating mechanism $\P_{W\vert Z}$, we do not have an analytic solution of~\eqref{eq:bayes}. Nevertheless, inspired by the Causal Effect Variational Autoencoder (CEVAE) \cite{louizos2017causal}, we will show how to approximate~\eqref{eq:bayes} for a high-dimensional proxy using VAEs in the next section.

\begin{table}[t]
\begin{minipage}[t]{0.47\textwidth}
\begin{algorithm2e}[H]
  \SetKwInOut{Input}{Input}
  \caption{Bandit with transportability for binary model.}
  \label{alg:binary}
  \Input{Prior dataset $D$}
  Estimate $p(w\vert z) \gets D$\\
  Estimate $p(y\vert x,z) \gets D$\\
  \For{$t\in \{1, \ldots, T\}$}{
    observe $w_t$\;
    Update $p^{\ast}(w)$\;
    Update $p^{\ast}(z)$ by~\eqref{eq:p(z)}\;
    \For{$x \in \{0,1\}$}{
      Compute expected reward $\mathbb{E}_{\pi^{\ast}}[Y\vert w_t, do(x)]$ by~\eqref{eq:TF2}\;
    }
    $x_t \gets \mathrm{\argmax}_{x}\mathbb{E}_{\pi^{\ast}}[Y\vert w_t,do(x)]$\;
    take $x_t$ and get $y_t$\;
  }
\end{algorithm2e}
\end{minipage}
%\end{table}
%\begin{table}[t]
\begin{minipage}[t]{0.5\textwidth}
\begin{algorithm2e}[H]
%\begin{multicols}{2}
  \SetKwInOut{Input}{Input}\
  \caption{Bandit with transportability for high-dimensional proxy.}

  \label{alg:general}
  \Input{Prior dataset $D$}
  Estimate $p_{\theta_1}(w\vert z) \gets D$\\
  Estimate $p_{\theta_2}(y\vert x,z) \gets D$ 

  Initialize replay buffer $\mathcal{B}$ with capacity $N$\\
  Initialize $\phi$ for the encoder $q_{\phi}(z \vert w)$ 

  \For{$t\in \{1, \ldots, T\}$}{
    observe $w_t$\;
    $\hat{z_t}$ $\gets$ $q_{\phi}(z \vert w_t)$\;
    $x_t$ $\gets$ $\mathrm{\argmax}_{x} p_{\theta_2}(y \vert x_, \hat{z_t})$\;
Take  $x_t$, get $y_t$, store ($w_t, x_t, y_t$) in $\mathcal{B}$\;
    \If{t is a gradient step}{
    $\mathbf{W}^M$ $\gets$ Sample a random mini-batch of $M$ data points from $\mathcal{B}$\;
    Perform gradient step on~\eqref{eq:ELBO transfer} with respect to $\theta_1$, $\theta_2$ and $\phi$ using $\mathbf{W}^M$\;
    }
  }
  % \end{multicols}
\end{algorithm2e}
\end{minipage}
\end{table}

\subsubsection{Restoring Unknown Posterior for High-dimensional Proxy}
\label{sec:restore posterior}
Recall that we start the derivation of~\eqref{eq:p(z)} from the marginal distribution of the proxy $W$. We then derive the solution for the unknown posterior. However, it is generally intractable to compute the marginal for a high-dimensional $W$ \citep{blei2017variational}. Therefore, we propose leveraging a VAE to restore the marginal distribution based on the \textit{latent-variable model}. It maximizes the marginal distribution of the observation by maximizing its evidence lower bound (ELBO)
\begin{equation}
  \label{eq:ELBO}
  \mathcal{L}_{\mathrm{VAE}}=\mathbb{E}_{z\sim q_{\phi}(z\vert w)}[\mathrm{log}p_{\theta}(w\vert z) - \mathrm{KL}(q_{\phi}(z \vert w)\vert \vert p(z))],
\end{equation}  
where $\mathrm{KL}(\cdot\vert \cdot)$ is the  Kullback-Leibler (KL) divergence \citep{kullback1951information}. Note, $p(z)$ here denotes the prior we set for $z$. Two neural networks are used to parameterize the approximated posterior $q_{\phi}(z \vert w)$ and the conditional $p_{\theta}(w\vert z)$, denoted by the encoder and the decoder, respectively.

The encoder can be used to approximate the posterior in~\eqref{eq:bayes}, and we could directly transfer a decoder from the source domain as the mechanism that generates $W$ from $Z$ is invariant across the domains given Assumption~\ref{covariate shift}. Namely, we can have the optimal decoder $p_{\theta}(w\vert z)=p^{\ast}(w\vert z)$ in the target domain and this can drive the encoder $q_{\phi}(z \vert w)$ to approximate the true $p^{\ast}(z\vert w)$.

One issue with this approach is that we do not know the prior distribution $p(z)$ in~\eqref{eq:ELBO}. Typically, the prior is chosen as $\mathcal{N}(0, \mathrm{I})$, which drives the encoder to map the high-dimensional input $w$ to a range close to $\mathcal{N}(0, \mathrm{I})$. However, this may be problematic in our setting, as the true posterior does not necessarily match with $\mathcal{N}(0, \mathrm{I})$. 

Nevertheless, $\mathcal{N}(0, \mathrm{I})$ could be an effective prior for the encoder to start with if we apply less regularization in the latent space, i.e., a smaller coefficient on the KL divergence term. This allows the encoder to learn a more flexible posterior. When combined with an optimal decoder, this approach encourages the encoder to approximate the true posterior $p^{\ast}(z\vert w)$ more effectively, rather than being heavily constrained by an overly simplistic assumption about the latent space. 

This approach leads us to leverage the objective function of $\beta$-VAE \citep{higgins2017beta}, which adds a constraint coefficient $\beta$ on the KL term in~\eqref{eq:ELBO},
\begin{equation}
  \label{eq:beta-vae}
  \mathcal{L}_{\beta\text{-VAE}}=\mathbb{E}_{z\sim q_{\phi}(z\vert w)}[\mathrm{log}p_{\theta}(w\vert z) - \beta \mathrm{KL}(q_{\phi}(z \vert w)\vert \vert p(z))].
\end{equation}

Notice that $\beta$ is set to be greater than 1 in the original $\beta$-VAE since the objective there is to learn a disentangled representation. In our approach, $\beta$ is set to be smaller than 1 as we want the encoder to learn a flexible posterior.

So far, we have made the design choice of restoring the unknown posterior. We still need to estimate the $w$-specific causal effect in~\eqref{eq:TF2}, which we show in the next section.

\begin{remark}
\label{unidentifiability}
    Restoring the unknown posterior $p^{\ast}(z\vert w)$ entirely is impossible without any inductive bias due to the unidentifiability of VAEs \citep{locatello2019challenging, rissanen2021critical}. % That is,  distinguishing two equivalent generative models from the marginal distribution of the observation is impossible without any inductive bias when the latent dimension is greater than 1.  
    However, under certain assumptions, the identifiability of the latent distribution $P(Z)$ can be achieved \citep{kivva2022identifiability}. Let $f$ be the function that generates $W$ from $Z$, the weakest assumptions that guarantee the identifiability of $P(Z)$ up to an affine transformation are (1) $P(Z)$ is a Gaussian mixture model; (2) $f$ is a piecewise affine function; (3) $f$ is weakly injective.
\end{remark}

\subsubsection{Causal Effect VAE with Transportability}

The transport formula in~\eqref{eq:TF2} suggests that both $p(y\vert z,x)$ and $p^{\ast}(z\vert w)$ are required to estimate the $w$-specific causal effect $p^{\ast}(y \vert w, do(x))$. The encoder $q_\phi(z\vert w)$ in~\eqref{eq:beta-vae} can be used to approximate $p^{\ast}(z \vert w)$, and we can use another neural network to estimate $p(y\vert z, x)$ from the prior data. Yet, in practice, we may still need to adapt the estimated $p(y\vert z,x)$ in the target domain since the distribution of $z$ is potentially different from the source domain. One way to adapt this neural network in the target domain is to jointly train it with the $\beta$-VAE by adding it as another decoder. Therefore, we obtain a new objective function,
\begin{equation}
  \label{eq:ELBO transfer}
  \mathcal{L}_{\mathrm{Trans}}=\mathbb{E}_{z\sim q_{\phi}(z\vert w)}[\mathrm{log}p_{\theta_1}(w\vert z) + \mathrm{log}p_{\theta_2}(y\vert x, z) - \beta \mathrm{KL}(q_{\phi}(z \vert w)\vert \vert p(z))].
\end{equation}
Note that the two decoders $p_{\theta_1}(w\vert z)$ and $p_{\theta_2}(y\vert x, z)$ can be estimated from the prior dataset $D$, and both of them encourage the encoder to approximate the true posterior.

A remaining issue is how we can train this VAE in the target domain, given that we have only one observation at each step. The solution is using a \textit{replay buffer} \citep{mnih2013playing} to store the observation pair ($w,x,y$) and train the VAE by sampling a batch from it for a fixed frequency. The concrete algorithm is given in Alg.~\ref{alg:general}, and the architecture is shown in Fig.~\ref{fig:arch}. We provide an implementation at \url{https://github.com/Marksmug/CEVAE-Transportability}.

Since Alg.~\ref{alg:general} uses stochastic gradient descent for optimizing~\eqref{eq:ELBO transfer}, we use the \textit{parameter counting argument} \citep{bengio2019meta} to explain the benefit of our method over naively transferring $p(y\vert w, do(x))$ from the source domain. In practice, the latter can be considered as transferring a full autoencoder in Fig.~\ref{fig:arch} from the source domain.
\begin{corollary} \quad%how to proply refer%
\label{corol1}
    Given the setting in Section~\ref{sec:2.3} and Assumption~\ref{covariate shift}. Let $\boldsymbol{V}=\{ X, Y, Z, W\}$ be the set of all variables, and $P(\boldsymbol{V})$ be the joint distribution learned from the source domain by the two decoders in Fig.~\ref{fig:arch} and $P'(\boldsymbol{V})$ be the one learned by a full autoencoder in Fig.~\ref{fig:arch}. The expected gradient w.r.t.\ the parameter $\theta_1$ of the log-likelihood under the target domain will be zero for the former one and non-zero for the latter, i.e.,
  \begin{equation}
    \mathbb{E}_{\pi^{\ast}}\left[\frac{\partial \mathrm{log}P'(\boldsymbol{V})}{\partial \theta_1}\right] \neq \mathbb{E}_{\pi^{\ast}}\left[\frac{\partial \mathrm{log}P(\boldsymbol{V})}{\partial \theta_1}\right] = 0.
  \label{eq:gradient}
  \end{equation}
\end{corollary}
Corolarry~\ref{corol1} (see proof in Appendix \ref{Appd:proof for P7}) formally justifies the adaptation benefit of our approach over naively transferring $p(y\vert w, do(x))$ using CEVAE-like latent variable models.  It suggests that our approach guided by the transport formula \eqref{eq:TF2} needs to adapt fewer parameters in the target domain than the naively transferring approach. This consequently affects the number of samples required for the adaptation in the target domain.
\begin{figure}[t]
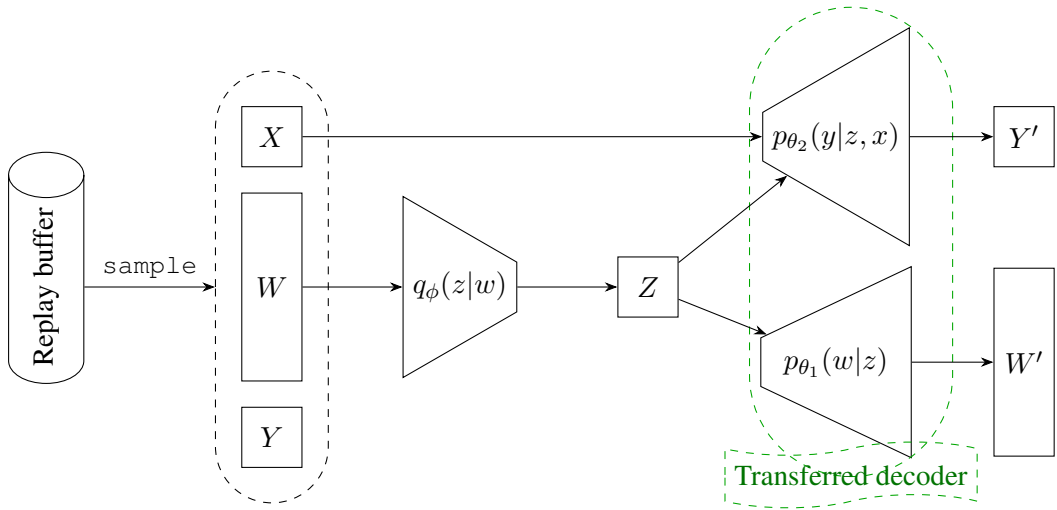

\vspace*{-1cm}
  \centering
\includeteximage{tikz/achitect.tikz}
  \caption{Architecture of VAE with transportability.}
  \label{fig:arch}
\end{figure}
\begin{remark}
    The architecture and the objective function of Alg.~\ref{alg:general} is similar to CEVAE, which estimates $p(w,z,x,y)$ from the observation of $(w,x,y)$. The objective of the two methods is the same. Both of them estimate the causal effect of the action $x$ given the observation of the proxy $w$. The differences between them are (1) CEVAE is trained in an offline setting, while Alg.~\ref{alg:general} is trained online; (2) instead of restoring $p(w,z,x,y)$, Alg.~\ref{alg:general} only restores $p(z\vert w)$, reducing the number of parameters required to adapt (as shown in Corollary~\ref{corol1}) compared to CEVAE; (3) Alg.~\ref{alg:general} does not restore $p(x\vert z)$ using another network, while CEVAE does.     
\end{remark}
\section{Evaluation}
In this section, we evaluate the proposed method on both synthetic and semi-synthetic datasets. Specifically, we first evaluate the method described in Alg.~\ref{alg:binary} on a simulated binary model. 
We then extend to high-dimensional proxy variables and demonstrate the effectiveness of our method in real-world datasets. In particular, we modify the Infant Health and Development Program (IHDP) dataset~\citep{hill2011bayesian} and test our method on it. Moreover, we want to investigate the performance of our method on an even higher dimension of the proxy. Inspired by \cite{rissanen2021critical}, we employ the image dataset MNIST~\citep{lecun1998mnist} to create a setting where the proxy is an image.
%Besides, we provide an additional synthetic experiment using a linear Gaussian model in Appendix~\ref{Appd: linear-Gaussian}.

%For the experiment of the binary model, we run multiple simulations and use the average cumulative regret in~\eqref{eq:expected regret} in the target domain as our evaluation metric. While for the one with high-dimensional proxy, we use the averaged total cumulative regret against gradient steps in the target domain as the evaluation metric.
For the binary model experiment, we run multiple simulations and evaluate the performance using the average cumulative regret in~\eqref{eq:expected regret} within the target domain. For the high-dimensional proxy case, we use averaged total cumulative regret against gradient steps as the evaluation metric.We refer to our algorithms as causal agents. There are three main questions we want to answer: (1) Does our method avoid the negative transfer in the target domain? (2) Does our method improve the sample efficiency in the target domain? (3) Can our method restore the unknown posterior $p^{\ast}(z \vert w)$ in the target domain? To answer the first one, we compare our method with baselines that transfer $\mathbb{E}[Y\vert x, w]$ from the prior dataset. For the second, we compare our method with a baseline without access to the prior dataset. For the third one, we compare the posterior $p^{\ast}(z\vert w)$ learned by our method with the true distribution of the latent context $\PS_Z$.

\subsection{Synthetic Binary Model}
\label{sec:binary model}
We begin by evaluating Alg.~\ref{alg:binary} on a binary model. We consider a bandit setting characterized by the causal graph shown in Fig.~\ref{fig:targetG}, involving proxy, context, action, and reward variables $W, Z, X, Y \in \{0,1\}$, respectively. We set the optimal action as opposite to the context. Details of the data-generating process are provided in Appendix~\ref{Appd:experiment details}. 

We provide 1000 prior samples collected from the source domain and average the cumulative regret in the target domain over 100 simulations. We compare our method with standard baselines, including Upper Confidence Bound (UCB) \citep{auer2002using} and Thompson Sampling (TS) \citep{thompson1933likelihood}.  We condition both algorithms on $w$, referring to them as Context-UCB (CUCB) and Context-TS (CTS). We refer to their counterparts that directly transfer $\mathbb{E}[Y\vert w, do(x)]$ from the prior dataset, as CUCB$^{-}$ and CTS$^{-}$, respectively.

The first row of Fig.~\ref{fig:binaryResults} shows the cumulative regret comparisons for all agents under four different settings, where the optimal policy is different across domains. We observe that both the CUCB$^{-}$ and CTS$^{-}$ agents exhibit significantly worse performance than their counterparts without access to prior data. This observation suggests that both agents suffer from negative transfer in these settings. In contrast, the causal agent, which transfers knowledge explicitly via the transport formula~\eqref{eq:TF2}, successfully avoids negative transfer. However, due to the relatively simple binary model, the causal agent does not demonstrate significantly better performance than the CTS agent, which learns directly from the target domain. The second row shows the estimation of $p^{\ast}(z)$ (blue) based on~\eqref{eq:p(z)} and the distribution of $Z$ in the source domain (solid black) as well as the target domain (dashed black). It can be seen the estimation converges to $p^{\ast}(z)$ in all settings. 

\begin{figure}
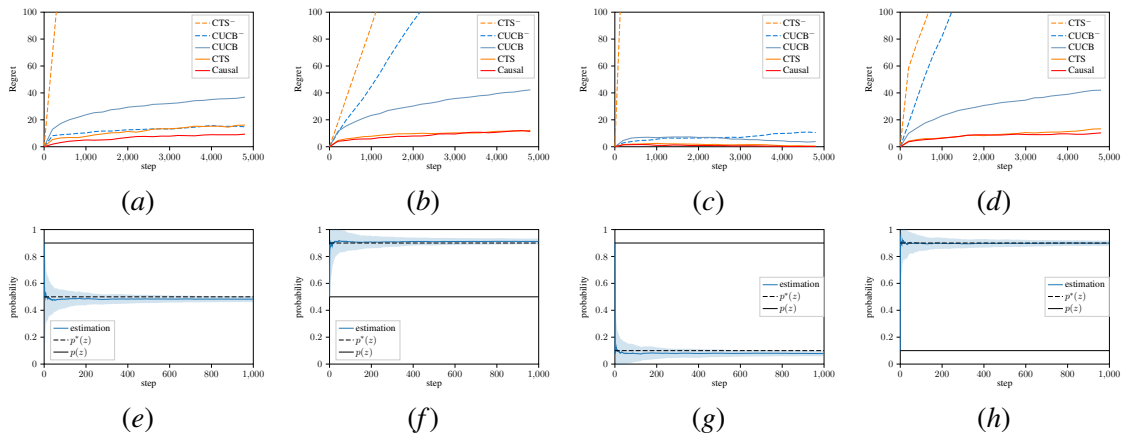

  \centering
\subfigure{
\includeteximage[width=0.23\textwidth]{tikz/binary_0.9_0.5.tex}
\label{fig:MNIST_regret1}
}
\subfigure{
\includeteximage[width=0.23\textwidth]{tikz/binary_0.5_0.9.tex}
\label{fig:MNIST_regret1}
}
\subfigure{
\includeteximage[width=0.23\textwidth]{tikz/binary_0.9_0.1.tex}
\label{fig:MNIST_regret1}
}
\subfigure{
\includeteximage[width=0.23\textwidth]{tikz/binary_0.1_0.9.tex}
\label{fig:MNIST_regret1}
}
\\
\subfigure{
\includeteximage[width=0.23\textwidth]{tikz/binary_dist_z_0.9_0.5.tex}
\label{fig:MNIST_regret1}
}
\subfigure{
\includeteximage[width=0.23\textwidth]{tikz/binary_dist_z_0.5_0.9.tex}
\label{fig:MNIST_regret1}
}
\subfigure{
\includeteximage[width=0.23\textwidth]{tikz/binary_dist_z_0.9_0.1.tex}
\label{fig:MNIST_regret1}
}
\subfigure{
\includeteximage[width=0.23\textwidth]{tikz/binary_dist_z_0.1_0.9.tex}
\label{fig:MNIST_regret1}
}

  \caption{Cumulative regret and the distribution comparison for the binary model. \capt{Each column shows the result of one setting, where context $Z$ follows different Bernoulli distributions. The top row shows the cumulative regret averaged over 100 simulations for Alg.~\ref{alg:binary} and the baselines. The bottom row shows the convergence of the estimation (blue) to the true distribution of $Z$ in the target domain (dashed black) and the distribution of $Z$ in the source domain (solid black)}.}
  \label{fig:binaryResults}
\end{figure}

\subsection{Semi-synthetic Data}
\label{sec:semi-synthetic}
In this section, we evaluate Alg.~\ref{alg:general} on two real-world datasets: IHDP and MNIST. IHDP is typically used to evaluate individual-level causal inference, and MNIST is a high-dimensional image dataset.  Following \cite{rissanen2021critical}, we modify both datasets to make the data-generating process compatible with our assumption; details can be found in Appendix~\ref{Appd:experiment details}. 

For both experiments, we use the total cumulative regret after 1000 steps as our evaluation metric, averaged over 5 random seeds. We define the following three agents:
(1) \textbf{VAE (prior)}: we pre-train a VAE on the prior data by reconstructing $W$ as well as predicting $Y$, and start the for loop in Alg.~\ref{alg:general}. This agent directly transfers the interventional distribution $p(y \vert w, do(x))$ from the source domain according to the transport formula~\eqref{eq:TF2}~\footnote{Since the encoder that learns the posterior $p(z\vert w)$ in the source domain is transferred.}; %We expect CEVAE(prior) agent to show negative transfer in the target domain as the optimal action is the opposite in the target domain;
(2) \textbf{VAE}: we initialize a random VAE and start the for loop in Alg.~\ref{alg:general} without any prior knowledge;
(3) \textbf{Causal}: we pre-train two decoders on the prior data as shown in Alg.~\ref{alg:general}, which are two neural nets with one mapping context $z$ to proxy $w$ and the other one mapping context $z$ and action $x$ to reward $y$.
For fail comparison, we use the same architecture and objective function for all agents. The details of VAEs used in the experiments can be found in Appendix~\ref{Appd:model details}. 

\subsubsection{Proxy IHDP dataset}
\begin{figure}[t]
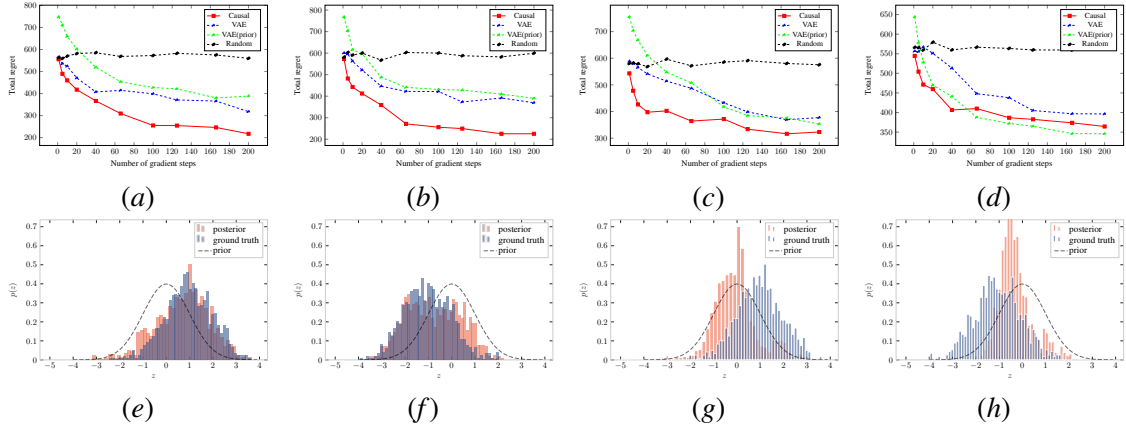

\centering
\subfigure{
\includeteximage[width=0.23\textwidth]{tikz/IHDP_-1_1.tex}
\label{fig:MNIST_regret1}
}
\subfigure{
\includeteximage[width=0.23\textwidth]{tikz/IHDP_1_-1.tex}
\label{fig:MNIST_regret2}
}
\subfigure{
\includeteximage[width=0.23\textwidth]{tikz/MNIST_-1_1.tex}
\label{fig:MNIST_regret3}
}
\subfigure{
\includeteximage[width=0.23\textwidth]{tikz/MNIST_1_-1.tex}
\label{fig:MNIST_regret4}
}
\\
\subfigure{
    \includeteximage[width=0.23\textwidth]{tikz/IHDP_dist_z_-1_1.tex}
}
\subfigure{
    \includeteximage[width=0.23\textwidth]{tikz/IHDP_dist_z_1_-1.tex}
}
\subfigure{
    \label{fig:MNIST_dist_z_-1_1}
    \includeteximage[width=0.23\textwidth]{tikz/MNIST_dist_z_-1_1.tex}
}
\subfigure{
    \label{fig:MNIST_dist_z_1_-1}
    \includeteximage[width=0.23\textwidth]{tikz/MNIST_dist_z_1_-1.tex}
}
  \caption{Total cumulative regrets and distributions comparison for IHDP and MNIST dataset. \capt{The first two columns are results for the IHDP dataset, and the last two are for the MNIST dataset. More results can be found in Appendix~\ref{Appd: more results}.}}
  \label{fig:IHDPResults}
\end{figure}
The IHDP dataset was created to study the individual causal effects of intensive child care on future cognitive test scores. It collects 25 different measurements of the children and their mothers, some continuous and some categorical, and the corresponding binary treatment they take in a randomized experiment as well as the continuous synthetically generated score. Given the limited size of the original dataset (747 samples), we train a VAE on the measurements with the latent dimension to be 1 and use it to generate more samples of the measurements. To make the data-generating process compatible with our assumption, we use the latent variable of the VAE as the latent context variable $Z$ and the generated measurement as the proxy variable $W$. Additionally, we use a simple function to regenerate the score $Y$ for each sample we generated from the VAE, such that the optimal action is 0 if $z < 0$ and 1 otherwise.

We want to compare the sample efficiency of the three agents, which can be shown by the total cumulative regret over the number of gradient steps. In this case, fewer gradient steps and smaller total cumulative regret means better sample efficiency. We provide 1000 samples from the source domain, and show our results in the first two columns of Fig.~\ref{fig:IHDPResults} with $\P_Z=\mathcal{N}(1,1)$, $\PS_Z=\mathcal{N}(-1,1)$ and $\P_Z=\mathcal{N}(-1,1)$, $\PS_Z=\mathcal{N}(1,1)$, respectively. Here, the causal agent shows a constantly higher sample efficiency than other agents. It overcomes the negative transfer that shows in the VAE (prior) agent, and it shows a better sample efficiency than the VAE agent. 

Additionally, we plot the posterior distribution (blue) learned by the encoder and the true $\PS_Z$ (orange) in the second row of Fig.~\ref{fig:IHDPResults}. Recall~\eqref{eq:ELBO transfer}, where the prior we set for the encoder to learn is $\mathcal{N}(0,1)$. However, the encoder instead learns a shift to the true $\PS_Z$. This shows that our method can approach the true $\PS_Z$ even without observing it. 

\subsubsection{Proxy MNIST Dataset}

We use the MNIST dataset to evaluate the performance of our method on image proxies. To make the data-generating process match with the causal graphs, we train a generative adversarial network (GAN)  \citep{goodfellow2014generative} on the MNIST dataset to generate images from a three-dimensional Gaussian noise. Its first dimension is then used as the context $Z$, and we generate binary $X$ and real-valued $Y$ from the context $Z$ for each sample. The function that generates $y$ from $z$ and $x$ is similar to the one used in the IHDP dataset. The optimal action for $z<0$ is 0 and 1 otherwise. 

Similar to the settings in the IHDP dataset, we divide the dataset according to the distribution of $Z$. In the first setting, $\P_Z$ is set to be $\mathcal{N}(1,1)$ and $\PS_Z$ is $\mathcal{N}(-1,1)$ and vice versa in the second setting. We train all VAEs with a three-dimensional latent and choose the first dimension as the restored $z$. The prior dataset consists of 1000 samples from the source domain. We show the results of the two settings in the last two columns in Fig.~\ref{fig:IHDPResults}. As expected, the VAE (prior) agent shows a negative transfer in both cases without updating in the target domain. As we increase the gradient steps, VAE (prior) adapts quickly in the second setting while relatively slow in the first setting. It takes 40 gradient steps to achieve a performance slightly better than the random agent. The causal agent, on the other hand, adapts quickly in both settings and constantly shows a lower regret than VAE agent, which has no prior knowledge of the source domain. 

However, the posterior distributions learned by the encoder do not consistently converge to the true $\PS_Z$. In the first setting, the posterior aligns more with $\P_Z$ instead of $\PS_Z$ (see Fig.~\ref{fig:MNIST_dist_z_-1_1}). In the second setting, the posterior shifts slightly towards $\PS_Z$ (see Fig.~\ref{fig:MNIST_dist_z_1_-1}). We hypothesize this is because of the unidentifiability of VAEs discussed in Remark~\ref{unidentifiability}. Although the transferred decoder is assumed to serve as an inductive bias in the target domain, it is insufficient to restore the latent variable distribution with the dimension of three in this experiment.

\section{Conclusion}
In this paper, we addressed the challenge of negative transfer in a latent contextual bandit problem with a high-dimensional proxy, leveraging causal transportability theory. We begin by introducing a novel algorithm for the binary case, showing how to effectively transfer knowledge from a source to a target domain via the transport formula. We then extend this to handle high-dimensional proxies using VAEs.  Our experiments on both synthetic and semi-synthetic datasets demonstrate that our method successfully avoids negative transfer and improves sample efficiency under the covariate shift on the latent context. These findings highlight the effectiveness of using causal transportability to tackle distribution shifts between environments and enhance the adaptability of bandit algorithms. Extending this framework to reinforcement learning, where a Markov decision process is considered to govern the data-generating process, could be an intriguing avenue for future work. 
% Future work could explore the extension of these methods to other reinforcement learning settings or investigate their applicability in more complex real-world scenarios.
% Acknowledgments---Will not appear in anonymized version
\acks{This work was supported by the Research Council of Finland Flagship programme: Finnish Center for Artificial Intelligence FCAI.}

\bibliography{reference.bib}
\newpage
\appendix
\section{Related Work}
\label{Appd:related work}
In this paper, we propose to use the transport formula to address the negative transfer in latent contextual bandits under the assumption of distribution shift on the latent context. Also, we assume a high-dimensional proxy variable and employ VAEs to restore the latent context from it. In this section, we relate our contributions to the existing literature.

\noindent \textbf{Transfer Learning in Bandits}. Our work is closely aligned with the framework of transfer learning in bandits, which aims to accelerate the convergence to the optimal actions in the target domain by transferring knowledge from an offline dataset collected from related environments. Existing works vary in their assumptions regarding the ``differences" between environments. For example, \cite{deshmukh2017multi} assume the reward distribution for each arm given the context is different and improves the estimation of the reward distributions for each arm by leveraging the similarities in the context-to-reward mapping. \cite{liu2018transferable} address potentially different context and action spaces and introduce a method to learn a translation matrix for aligning feature spaces between source and target domains. While they assume a single unknown reward parameter shared between domains, \cite{labille2021transferable} relax this assumption by considering different unknown reward parameters across domains. Combining with LinUCB, they propose leveraging historical data from similar domains to initialize new arm parameters, using a similarity measure and a decaying weight to prioritize relevant observations.  \cite{zhang2017transfer} address the presence of unobserved confounders that simultaneously influence both the rewards and actions. Under the assumption of shifts in reward distributions, they propose to derive causal bounds from prior datasets, utilizing these bounds to inform and optimize exploration and exploitation within the target domain. Although these works explore various assumptions regarding domain differences, none of them considers changes in the distribution of the context. The most closely related work to ours is by \cite{cai2024transfer}.
\cite{cai2024transfer} explore transfer learning for contextual bandits under covariate shift, \ie the distribution of context is different across domains. However, their approach involves partitioning the covariate space into bins, which may not be feasible for high-dimensional covariates such as the proxy in our problem setting. Moreover, their setting does not incorporate an underlying latent context, as is the case in our latent contextual bandit framework.

\noindent \textbf{Latent Contextual Bandits}. Latent contextual bandits were initially introduced by \cite{zhou2016latent} and later extended by \cite{hong2020latent}, who employ offline-learned models that provide an initial understanding of the latent states. However, their problem setting differs from ours in two key aspects: (1) their reward function depends on the context, action, and latent state, whereas our reward depends only on the action and latent state; and (2) they do not account for distributional shifts in the context or latent state. More closely related to our setting, \cite{sen2017contextual} consider a latent confounder associated with the observed context, and the reward depends solely on the latent confounder and action. Unlike our approach, they address the problem in an online setting by factorizing the observed context-reward matrix into two low-dimensional matrices, with one capturing the relationship between contexts and latent confounders and the other between rewards and latent confounders. Although this work defines their problem setting from a causal perspective, they do not apply any theory from causal inference to derive their method. %adding the explanation. 

\noindent \textbf{Bandits and Causal Inference}. The connection between bandit problems and causal inference is well-established. \cite{bareinboim2015bandits} introduce the problem of bandits with unobserved confounders and leverage counterfactual estimation from causal inference to reduce the cumulative regret. Building on this, \cite{forney2017counterfactual} further exploits counterfactual reasoning to generate counterfactual data, integrating it with the observational data and interventional data to learn a good policy. Moreover, there is extensive literature using causal graphs to define the structure of bandits under the name of causal bandits \citep{lattimore2016causal, lee2018structural,lee2019structural, lu2021causal, bilodeau2022adaptively, yan2024causal}. While those works exploit knowledge in a single environment, \cite{bellot2024transportability} propose a general transfer learning framework for causal bandits, using posterior approximations on the causal relations encoded in the selection diagram, grounded in transportability theory. We realize their approach of learning the invariant causal relations aligns with our method, and we position our work as an extension of their framework within the specific context of latent contextual bandits. Further, we relax their assumption of discrete SCMs by considering a high-dimensional proxy variable and a continuous latent context.
\section{Proofs for Section \ref{sec:challenge on unobserved context} }
\label{appd:props}
\begin{proposition*}\ref{prop1} \quad%how to proply refer%
  For an offline contextual bandit under covariate shift, the causal effect $p(y\vert z, do(x))$, where $z\in \mathcal{Z}$, $x \in \mathcal{X}$ and $Y \in \mathbb{R}$ is directly transportable. Additionally, the transport formula of $p^{\ast}(y\vert z, do(x))$ can be written as 
  \begin{equation*}
    p^{\ast}(y\vert z, do(x))=p(y\vert z, x).
  \end{equation*}
\end{proposition*}
\begin{proof}
  Given the selection diagram $\mathcal{D}$ in Fig.~\ref{fig:SD_}, the proof is two-fold. First, we prove the causal effect $p^{\ast}(y\vert z, do(x))$ is directly transportable from the source domain $\pi$ to the target domain $\pi^{\ast}$. Second, we prove the transport formula can be written as $p^{\ast}(y\vert z, do(x))=p(y\vert z, x)$.

  We can rewrite $p^{\ast}(y\vert z,do(x))=p(y\vert z, s, do(x))$ by conditioning on the selection node $S$ to denote the distribution in the target domain \citep{pearl2011transportability}. Since we can observe $Z$, we can use \emph{d-separation} \citep{pearl2009causality} to prove $(Y \indep S \vert Z)_{\mathcal{D}_{\overline{X}}}$. This allows us to use the first rule of the do-calculus \cite{pearl2009causality}. Thus we have the transport formula $p^{\ast}(y\vert z,do(x))=p{(y\vert z,do(x))}$. Therefore, we prove the causal effect $p^{\ast}(y\vert z, do(x))$ is directly transportable from $\pi$ to $\pi^{\ast}$. 

  Following the result above, we can further apply the second rule of the do-calculus and obtain the transport formula $p^{\ast}(y\vert z,s,do(x))=p{(y\vert z, x)}$ based on the fact $(Y \indep X \vert Z)_{\mathcal{D}_{\underline{X}}}$. Hence,~\eqref{eq:TF1} holds. 

\end{proof}

\begin{proposition*}\ref{prop2}\quad
  For the problem described in Section\ref{sec:2.3}, the $w$-specific causal effect $p^{\ast}(y\vert w,do(x))$ is not directly transportable and the transport formula of $p^{\ast}(y\vert w,do(x))$ can be written as
  \begin{equation*}
    p^{\ast}(y\vert w,do(x))=\sum_{z\in\mathcal{Z}}p(y\vert w,z)p^{\ast}(z\vert w).
  \end{equation*} 
\end{proposition*}
\begin{proof}
  Given the selection diagram $\mathcal{D}$ in Fig.~\ref{fig:SD} that encodes the domain discrepancy in the problem outlined in Section~\ref{sec:2.3}, we can rewrite $p^{\ast}(y\vert w,do(x))=p(y\vert w,s,do(x))$. By d-separation, we can write $p(y\vert w,s,do(x))=p(y\vert w,do(x))$ if we have the condition $(Y \indep S \vert W)_{\mathcal{D}_{\overline{X}}}$. However, this condition is not satisfied as conditioning on the proxy $W$ still leaves the backdoor path $S\rightarrow Z\rightarrow Y$ open. Hence, we prove $p^{\ast}(y\vert w,do(x))\neq p(y\vert w,do(x))$ and the $w$-specific causal effect is not directly transportable.

  Next, we prove $p^{\ast}(y\vert w,do(x))=\sum_{z\in\mathcal{Z}}p(y\vert w,z)p^{\ast}(z\vert w)$. We can rewrite $p^{\ast}(y \vert w, do(x)) = p(y \vert  w, s, do(x)) $. By basic probability operations, we have
  \begin{align*}
    p(y \vert w, s, do(x)) = \sum_{z \in \mathcal{Z}}p(y \vert w, s, z, do(x))p(z \vert  w, s, do(x)).  
  \end{align*}
   Since $(S \indep Y \vert Z)_{\mathcal{D}_{\overline{X}}}$ and $(X \indep Z)_{\mathcal{D}_{\overline{X}}}$, by the first and third rule of the do-calculus, we have $p(y \vert w, s, do(x)) = \sum_{z \in \mathcal{Z}}p(y \vert w, z, do(x))p(z \vert w, s)$. Now, we can also observe that $(S \indep Y \vert Z)_{\mathcal{D}_{\overline{X}}}$. By the first rule of the do-calculus, we have $P(Y \vert w, s, do(x)) = \sum_{z \in \mathcal{Z}}p(y \vert z, do(x))p(z \vert w, s)$. Further, we have $(X \indep Y \vert Z)_{\mathcal{D}_{\underline{X}}}$, thus allowing us to apply the second rule of the do-calculus on $p(y\vert z,do(x))$ and obtain $p(y\vert z,do(x))=p(y\vert z,x)$. Finally, based on the definition of the selection node, we then have the transport formula $p^{\ast}(y\vert w,do(x))=\sum_{z\in\mathcal{Z}}p(y\vert w,z)p^{\ast}(z\vert w)$.
\end{proof}

\section{Experiment for Negative transfer}
\label{Appd: negative transfer}
We test on both a binary model and a linear Gaussian model to illustrate the negative transfer. Consider a binary model where $W$, $Z$, $X$, and $Y$ are binary variables. Let $\P_{Z} = \mathrm{Bern}(0.9)$ and $\PS_{Z} = \mathrm{Bern}(0.5)$. We set the optimal action $X$ opposite to $Z$. A Thompson sampling \citep{thompson1933likelihood} agent is trained to learn the $w$-specific causal effect $p(y \vert w, do(x))$ by identifying the optimal action conditioned on each $w \in \{0,1\}$; we refer to this agent as the Context-Thompson Sampling (CTS) agent. The learned $p(y \vert w, do(x))$ is then directly transferred to the target domain as if it were $p^{\ast}(y \vert w, do(x))$. We refer to this agent as CTS$^{-}$, and to the one that does not transfer knowledge from the source domain as CTS. The cumulative regrets are presented in Fig.~\ref{fig:negative cucb}. It can be observed that the CTS$^{-}$ agent exhibits linearly increasing cumulative regret, while the cumulative regret of the CTS agent trained directly in the target domain increases sub-linearly. 

Further, considering the high-dimensional nature of the proxy $W$, we conducted a simulation where $Z \in \mathbb{R}$, $W \in \mathbb{R}^{5}$ and $X,Y \in \{0,1\}$, using LinUCB \citep{li2010contextual} as the baseline agent. The distribution of $Z$ is $\mathcal{N}(8,1)$ and $\mathcal{N}(2,1)$ in the source domain and the target domain, respectively. The generating process from $Z$ to $W$ follows a linear-Gaussian transformation: $ w = a \odot z + b + \epsilon$, where $\epsilon \sim \mathcal{N}(0,1)$ and $\odot$ is the element-wise multiplication. We set the optimal action to be 1 if $z<5$, otherwise 0, \ie $y = \mathbbm{1}\{( z\geq 5 \wedge x = 0) \lor (z < 5 \wedge x = 1)\}$.  The result, shown in Fig.~\ref{fig:negative linucb}, demonstrates a similar trend of negative transfer.% Details of the data-generating mechanisms can be found in Appendix~\ref{Appd:experiment details}.

\begin{remark}
\label{connection to MAB with confounders}
 Interestingly, we notice the connection between the binary model and a transfer learning problem of \textit{bandits with unobserved confounders} proposed by \cite{bareinboim2015bandits}. In their setting, a muti-armed bandit (MAB) agent receives a predilection towards arm selection that is influenced by a hidden confounder, which also affects the reward. The agent then makes decisions based on the predilection. In their work, the counterfactual question ``What would have happened had the agent decided to play differently'' is asked. 
  By considering the transfer learning setting as defined in Section~\ref{sec:2.3} and interpreting the predilection towards the binary action as a proxy for an underlying binary confounder, the problem described in Section~\ref{sec:2.3} is established. 
\end{remark}

\section{Derivation of Eq.~\eqref{eq:p(z)}}
\label{Appd: Derivation for 6}
We start from the marginal distribution of $w=1$,
\begin{align*}
  p^{\ast}(w=1)&=\sum_{z\in \{0,1\}}p^{\ast}(w=1\vert z)p^{\ast}(z)\\
  &=p^{\ast}(w=1\vert z=0)p^{\ast}(z=0) + p^{\ast}(w=1\vert z=1)\PS(z=1)\\
  &=p^{\ast}(w=1\vert z=0)(1-p^{\ast}(z=1)) +p^{\ast}(w=1\vert z=1)p^{\ast}(z=1).
\end{align*}
Rearranging the equation, we obtain an expression for $p^{\ast}(z=1)$ as 
\begin{equation*}
  p^{\ast}(z=1) = \frac{p^{\ast}(w=1)-p^{\ast}(w=1\vert z=0)}{p^{\ast}(w=1 \vert z=1)-p^{\ast}(w=1\vert z=0)}.
\end{equation*} 
Note, in the assumption of covariate shift, we have $\PS_{W\vert Z}=\P_{W\vert Z}$, therefore, we have 
\begin{equation*}
  p^{\ast}(z=1) = \frac{p^{\ast}(w=1)-p(w=1\vert z=0)}{p(w=1 \vert z=1)-p(w=1\vert z=0)}.
\end{equation*} 
\section{Proof for Corollary.~\ref{corol1}}
In this section, we provide the proof for Corollary~\ref{corol1}. We first state the corollary again for convenience.
\begin{corollary*} \ref{corol1}\quad%how to proply refer%
    Given the setting in Section.~\ref{sec:2.3} and Assumption.~\ref{covariate shift}, let $\boldsymbol{V}=\{ X, Y, Z, W\}$ be the set of all variables, $\theta_1$ be the parameter that parameterize the mechanism generating $W$ from $Z$ and $P(\boldsymbol{V})$ be the joint distribution learned from the source domain $\pi$ by the two decoders in Fig.~\ref{fig:arch} and $P'(\boldsymbol{V})$ be the one learned by both encoder and decoders in Fig.~\ref{fig:arch}, the expected gradient w.r.t the parameter $\theta_1$ of the log-likelihood under the target domain $\pi^{\ast}$ will be zero for the former one and non-zero for the latter 
  \begin{equation}
    \label{eq:gradient_copy}
    \mathbb{E}_{\pi^{\ast}}\left[\frac{\partial \mathrm{log}P'(\boldsymbol{V})}{\partial \theta_1}\right] \neq \mathbb{E}_{\pi^{\ast}}\left[\frac{\partial \mathrm{log}P(\boldsymbol{V})}{\partial \theta_1}\right] = 0.
  \end{equation}
\end{corollary*}
\begin{proof}
    We prove the corollary in two steps. In the first step, we prove the equality in \eqref{eq:gradient_copy}. Then, we prove the inequality in \eqref{eq:gradient_copy}. We, therefore, derive two lemmas for the two steps, respectively.
\begin{lemma}
\label{lemma:1}
     Given the setting in Section~\ref{sec:2.3} and Assumption~\ref{covariate shift}, let $\boldsymbol{V}=\{ X, Y, Z, W\}$ be the set of all variables, $\theta_1$ and $\theta_2$ be the parameters that parameterize the mechanism generating $W$ from $Z$ and $Y$ from $Z$ and $X$. Furthermore, let $P(\boldsymbol{V})$ be the joint distribution learned from the source domain $\pi$ by the two decoders in Fig.~\ref{fig:arch}. The expected gradient w.r.t.\ the parameter $\theta_1$ of the log-likelihood under the target domain $\pi^{\ast}$ will be zero,
  \begin{equation}
    \mathbb{E}_{\pi^{\ast}}\left[\frac{\partial \mathrm{log}P(\boldsymbol{V})}{\partial \theta_1}\right] = 0.
  \end{equation}
\end{lemma}
\begin{proof}
Notice when we use two decoders to approximate $P_{\theta_1}(W\vert Z)$ and $P_{\theta_2}(Y\vert Z,X)$, the joint $P(\boldsymbol{V})$ can be factorized as 
    \begin{align}
    \label{eq:fac1}
    P(\boldsymbol{V}) = P(Z) P_{\theta_1}(W\vert Z) P_{\theta_2}(Y\vert Z,X) P(X\vert Z),
    \end{align}
    according to the casual graph in Fig.~\ref{fig:sourceG}.
    %Here, the last term is $P(X)$ instead of $P(X\vert Z)$ following Fig.~\ref{fig:sourceG}. The reason behind this factorization is that there is no model to learn the mechanism that generates $X$ from $Z$ in our approach, but only the decoder that learns the joint distribution $P(X, Y, Z)$, which can be factorized to $P(Y\vert Z,X) P(X) P(Z)$. 
With this factorization, the expected gradient can be simplified as
    \begin{align*}
    \mathbb{E}_{\pi^{\ast}} \left[ \frac{\partial \log P(\boldsymbol{V})}{d\theta_1} \right] 
    &= \mathbb{E}_{\boldsymbol{V} \sim P^*} \left[ \frac{\partial \log P(\boldsymbol{V})}{\partial\theta_1} \right] \\
    &= \mathbb{E}_{\boldsymbol{V} \sim P^*} \left[ \frac{\partial}{\partial\theta_1} \log P(Z) P_{\theta_1}(W\vert Z) P_{\theta_2}(Y\vert Z,X) P(X\vert Z) \right] \tag*{by~\eqref{eq:fac1}}\\
    &= \mathbb{E}_{\boldsymbol{V} \sim P^*} \left[ \underbrace{\frac{\partial}{\partial\theta_1} \log P(Z)}_{=0} + \frac{\partial}{\partial\theta_1} \log P_{\theta_1}(W\vert Z)\right.\\ 
    &\qquad\qquad+\left. \underbrace{\frac{\partial}{\partial\theta_1} \log P_{\theta_2}(Y\vert Z,X)}_{=0} + \underbrace{\frac{\partial}{\partial\theta_1} \log P(X\vert Z)}_{=0} \right] \\
    &= \mathbb{E}_{\boldsymbol{V} \sim P^*} \left[ \frac{\partial}{\partial\theta_1} \log_{\theta_1} P(W\vert Z) \right] \\
    &= \mathbb{E}_{\boldsymbol{V} \sim P^*} \left[ \frac{\partial}{\partial\theta_1} \log P^*_{\theta_1}(W\vert Z) \right] \quad \tag*{by Assumption 1} \\
    &= \mathbb{E}_{\boldsymbol{V} \sim P^*} \left[ \underbrace{\frac{\partial}{\partial\theta_1} \log P^*(Z)}_{=0} + \frac{\partial}{\partial\theta} \log P^*_{\theta_1}(W\vert Z) \right.\tag*{by~\eqref{eq:fac1}}\\
    &\qquad\qquad\left.+ \underbrace{\frac{\partial}{\partial\theta} \log P^*_{\theta_2}(Y\vert Z,X)}_{=0} + \underbrace{\frac{\partial}{\partial\theta} \log P^*(X\vert Z)}_{=0} \right]\\
    &= \mathbb{E}_{\boldsymbol{V} \sim P^*} \left[ \frac{\partial \log P^*(\boldsymbol{V})}{\partial\theta_1} \right] \\
    &= \sum_{\boldsymbol{v}}p^{*}(\boldsymbol{v})\frac{\partial \log p^*(\boldsymbol{v})}{\partial\theta_1}\tag*{expanding the expectation}\\
    &= \sum_{\boldsymbol{v}}p^{*}(\boldsymbol{v})\frac{\partial \log p^*(\boldsymbol{v})}{\partial p^*(\boldsymbol{v})}\frac{\partial p^*(\boldsymbol{v})}{\partial\theta_1}\tag*{chain rule}\\
    &= \sum_{\boldsymbol{v}}\frac{\partial p^*(\boldsymbol{v})}{\partial\theta_1}\\
    &= \frac{\partial}{\partial\theta_1}\underbrace{\sum_{\boldsymbol{v}} p^*(\boldsymbol{v})}_{=1}=0.
\end{align*}
Therefore, we prove the equality in \eqref{eq:gradient_copy}.
\end{proof}

\begin{lemma}
\label{lemma:2}
     Given the setting in Section~\ref{sec:2.3} and Assumption~\ref{covariate shift}, let $\boldsymbol{V}=\{ X, Y, W,Z\}$ be the set of all variables, $\phi$, $\theta_1$, and $\theta_2$ be the parameter that parameterize the posterior $P(Z\vert W)$, the mechanism generating $W$ from $Z$, and the mechanism generating $Y$ from $Z$ and $X$. Furthermore, let $P(\boldsymbol{V})$ be the joint distribution learned from the source domain $\pi$ by an autoencoder as shown in Fig.~\ref{fig:arch}. The expected gradient w.r.t.\ the parameter $\theta_1$ of the log-likelihood under the target domain $\pi^{\ast}$ will be non-zero,
  \begin{equation}
    \mathbb{E}_{\pi^{\ast}}\left[\frac{\partial \mathrm{log}P(\boldsymbol{V})}{\partial \theta_1}\right] \neq 0.
  \end{equation}
\end{lemma}
\begin{proof}
Notice when we use an autoencoder to approximate $P_{\phi}(Z\vert W)$, $P_{\theta_1}(W\vert Z)$, and $P_{\theta_2}(Y\vert Z, X)$, the joint $P(\boldsymbol{V})$ can be factorized as 
    \begin{align}
    \label{eq:fac2}
    P(\boldsymbol{V}) = P_{\phi}(Z\vert W) P_{\theta_1}(W\vert Z) P_{\theta_2}(Y\vert Z,X) P(X\vert Z).    
    \end{align}
    The first term on the R.H.S.\ of~\eqref{eq:fac2} is $P_{\phi}(Z\vert W)$ instead of $P(Z)$ in~\eqref{eq:fac1}, since we do not have an observation of $P(Z)$ but use an encoder to estimate $Z$ from $W$.   
With this factorization, the expected gradient can be simplified as 
    \begin{align*}
    \mathbb{E}_{\pi^{\ast}} \left[ \frac{\partial \log P(\boldsymbol{V})}{d\theta_1} \right] 
    &= \mathbb{E}_{\boldsymbol{V} \sim P^*} \left[ \frac{\partial \log P(\boldsymbol{V})}{\partial\theta_1} \right] \\
    &= \mathbb{E}_{\boldsymbol{V} \sim P^*} \left[ \frac{\partial}{\partial\theta_1} \log P_{\phi}(Z\vert W) P_{\theta_1}(W\vert Z) P_{\theta_2}(Y\vert Z,X) P(X\vert Z) \right] \tag*{by \eqref{eq:fac2}}\\
    &= \mathbb{E}_{\boldsymbol{V} \sim P^*} \left[ \frac{\partial}{\partial\theta_1} \log \frac{P_{\theta_1}(W\vert Z)P(Z)}{P(W)}P_{\theta_1}(W\vert Z) P_{\theta_2}(Y\vert Z,X) P(X\vert Z) \right]\\
    &= \mathbb{E}_{\boldsymbol{V} \sim P^*} \Bigg[
    \underbrace{\frac{\partial}{\partial\theta_1} \log P(Z)}_{=0} 
    + 2\frac{\partial}{\partial\theta_1} \log P_{\theta_1}(W\vert Z) 
    -\frac{\partial}{\partial\theta_1} \log P(W) \notag \\ 
 &\quad + \underbrace{\frac{\partial}{\partial\theta_1} \log P_{\theta_2}(Y\vert Z,X)}_{=0} 
    + \underbrace{\frac{\partial}{\partial\theta_1} \log P(X)}_{=0} 
    \Bigg]\\
    &= \mathbb{E}_{\boldsymbol{V} \sim P^*} \left[ 2\frac{\partial}{\partial\theta_1}\log P_{\theta_1}(W\vert Z) - \frac{\partial}{\partial\theta_1} \log P(W)\right] \\
    &= \mathbb{E}_{\boldsymbol{V} \sim P^*} \left[ 2\frac{\partial}{\partial\theta_1} \log P^*_{\theta_1}(W\vert Z) - \frac{\partial}{\partial\theta_1} \log P(W) \right] \tag*{by Assumption 1} \\
    &= -\mathbb{E}_{\boldsymbol{V} \sim P^*} \left[ \frac{\partial}{\partial\theta_1} \log P(W) \right] \neq 0
\end{align*}
The last equality arises from the last six steps in the proof of Lemma~\ref{lemma:1}, and the inequality arises from the fact that $P(W)$ is affected by the parameter $\theta_1$, which can be better understood from the marginalization of $W$ in the binary model (see Appendix~\ref{Appd: Derivation for 6}). Therefore, we prove the inequality in~\eqref{eq:gradient_copy}.  
\end{proof}
Combining Lemma~\ref{lemma:1} and Lemma~\ref{lemma:2}, we prove Corollary~\ref{corol1}.
\end{proof}
\label{Appd:proof for P7}
\section{Details on the data generating mechanisms}
\label{Appd:experiment details}
In this section, we explain the details of the data-generating mechanisms that we used in the experiments.
The code for running all experiments is available at \url{https://github.com/Marksmug/CEVAE-Transportability}.

\noindent \textbf{Synthetic binary model.} To generate data for both the source domain and the target domain, we use the following data-generating process that is compatible with the causal graph in Fig.~\ref{fig:sourceG}:
    $z \sim \mathrm{Bern}(c)$, $w \vert z \sim \mathrm{Bern}(0.8z + 0.1(1-z))$,
    $y = x \oplus z$, where $c \in [0, 1]$ is the parameter to control the distribution of $Z$ and $\oplus$ is the XOR operation. Let $c_s$ and $c_t$ be parameters in the source domain and the target domain, respectively. We consider four settings: (1) $c_s = 0.9$ and $c_t = 0.5$; (2) $c_s = 0.5$ and $c_t = 0.9$; (3) $c_s = 0.9$ and $c_t = 0.1$; (4) $c_s = 0.1$ and $c_t = 0.9$. 
\\\\
\noindent \textbf{Synthetic linear Gaussian model.}
To generate data for both the source domain and the target domain, we employ the following data-generating process: $z \sim \mathcal{N}(u,1)$, $ w = a \odot z + b + \epsilon$, $y = \mathbbm{1}\{( z\geq 5 \wedge x = 0) \lor (z < 5 \wedge x = 1)\}$, where $\odot$ is the element-wise multiplication.   
\\\\
\noindent \textbf{IHDP dataset.} We first train a data-generating VAE on the IHDP dataset to generate data compatible with the causal graph in Fig.~\ref{fig:sourceG}. The data-generating VAE has an encoder and a decoder that uses a 3-layer MLP with a layer width of 30 and ELU activation. Since the IHDP covariate includes continuous variables, categorical variables, and binary variables, we structure the output layer of the decoder in a way that outputs Gaussian distributions for continuous variables, softmax functions for categorical variables, and a logit for binary variables.  We sample \num{20000} $z$ from $\mathcal{N}(0,1)$ and use the decoder of the VAE to generate the corresponding proxy $w$. The reward $y$ is then generated by the function $y=zx + \epsilon$, where $\epsilon \sim \mathcal{N}(0,1)$. Actions $x$ are generated by firstly fitting a neural network from a scalar dimension of the measurement to the treatment (action variable) in the original data and then sampling the actions from the probability given by the neural network with latent contexts $z$ as the input. To divide the dataset into the source domain and the target domain, we use rejection sampling to sample two Gaussians from \num{20000} $z$ and their corresponding $w, x, y$.  For the target domain, we only have the pair $(z,w)$ and generate $y$ based on the current $z$ and the action $x$ for each step.
\\\\
\noindent \textbf{MNIST dataset.} We first train a GAN on the MNIST dataset to generate data compatible with the causal graph in Fig.~\ref{fig:sourceG}. The generator of the data-generating GAN takes a three-dimensional Gaussian noise as the input and uses four transpose convolutional layers with batch normalization and ReLU activation for all but the last layer, which employs a tanh activation. The discriminator takes an image as input and applies four convolutional layers, each with LeakyReLU activation, outputting a flattened probability score for each image, which ranges between 0 and 1, indicating the likelihood of the image being real or fake. When generating data, we sample \num{100000} latent values from $\mathcal{N}(0,I)$ and use the generator of the GAN to generate corresponding images as $w$. The first dimension of the latent value is then chosen as the latent context $Z$. The reward $y$ and action $x$ are generated by two functions, $y=zx + \epsilon$ and $x=\mathrm{Bern}(\sigma(z+0.5))$, where $\epsilon \sim \mathcal{N}(0,1)$ and $\sigma$ is the sigmoid function. To divide the dataset into the source domain and the target domain, we use rejection sampling to sample two Gaussians from \num{100000} $z$ and their corresponding $w,x,y$. For the target domain, we only use the pair $(z,w)$ and generate $y$ based on current $z$ and action $x$ for each step.

\section{Model details}
\label{Appd:model details}
In this section, we give the details of the architecture of VAEs used in Section~\ref{sec:semi-synthetic} and their training. The Adam optimizer is used for training in the experiments.
\subsection{IHDP}
\noindent \textbf{Architecture of VAEs.} The VAEs used in this experiment have the same architecture as depicted in Fig.~\ref{fig:arch}, with the difference of two separate neural networks for $p_{\theta_2}(y\vert z,x)$ that are chosen for different $x$ for each forward process. In total, all VAEs have four neural networks: $q_{\phi}(z\vert w)$, $p_{\theta_1}(w\vert z)$, $p_{\theta_2}(y \vert z, x=0)$, and $p_{\theta_2}(y \vert z, x=1)$. Similar to the data-generating VAE, we use a 3-layer MLP with a layer width of 30 for each neural network but a different activation function, ReLU. The latent dimension for all VAEs is set to be one as the data-generating VAE. The output layer for $p_{\theta_1}(w\vert z)$ follows the same structure as the data-generating VAE, and the two neural networks predicting $y$ output a Gaussian.

\noindent \textbf{Training.} For all VAEs, we use a replay buffer $\mathcal{B}$ with capacity 1000 and batch size $M=32$. We set $\beta=0.1$ to allow for a more flexible posterior and train all VAEs with frequency list $[1000, 200, 100, 50, 25,$ $ 15, 10, 8, 6, 5]$ corresponding to the number of gradient steps $[1, 5, 10, 20, 40,$ $ 66, 100, 125, 166, 200]$ with a learning rate of \num{0.005}.

\subsection{MNIST.} 

\noindent \textbf{Architecture of VAEs.} The VAEs used in this experiment have the same number of neural networks as the ones used in the IHDP dataset. The encoder $q_{\phi}(z \vert w)$ has four transpose convolutional layers and one linear layer, all with ELU activation. The decoder $p_{\theta_1}(w \vert z)$ has one linear layer and four transpose convolutional layers, all with ELU activation except the last transpose convolutional layer. The decoder for predicting $y$ is the same as the two networks used in the IHDP dataset.

\noindent \textbf{Training.} For all VAEs, we use a replay buffer $\mathcal{B}$ with capacity 1000 and batch size $M=32$. We set $\beta=0.1$ to allow for a more flexible posterior and train all VAEs with frequency list $[1000, 200, 100, 50, 25,$ $ 15, 10, 8, 6, 5]$ corresponding to the number of gradient steps $[1, 5, 10, 20, 40,$ $ 66, 100, 125, 166, 200]$ with a learning rate of \num{0.001}.
\section{More results for IHDP and MNIST dataset}
\label{Appd: more results}
\begin{figure}[h]
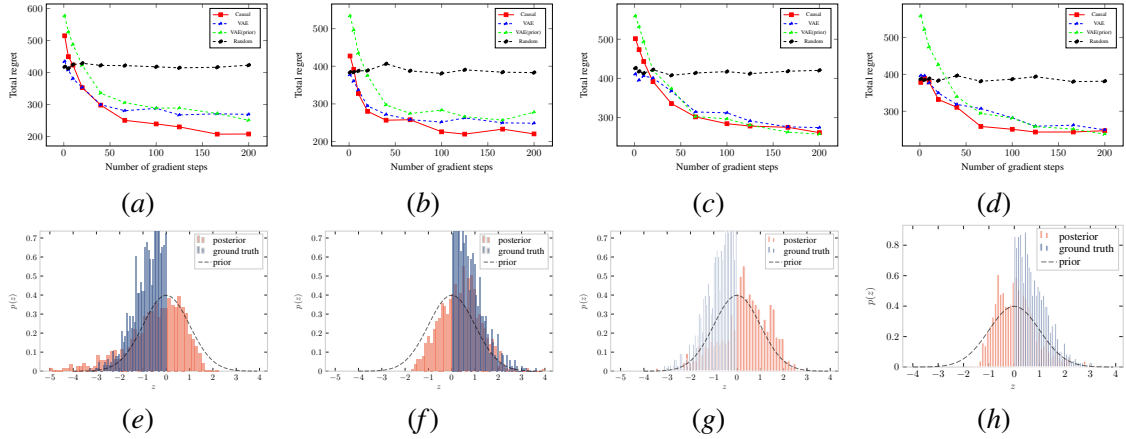

\centering
\subfigure{
\includeteximage[width=0.23\textwidth]{tikz/IHDP_+_-.tex}
\label{fig:IHDP_regret_extreme1}
}
\subfigure{
\includeteximage[width=0.23\textwidth]{tikz/IHDP_-_+.tex}
\label{fig:IHDP_regret2_extreme2}
}
\subfigure{
\includeteximage[width=0.23\textwidth]{tikz/MNIST_+_-.tex}
\label{fig:MNIST_regret_extreme1}
}
\subfigure{
\includeteximage[width=0.23\textwidth]{tikz/MNIST_-_+.tex}
\label{fig:MNIST_regret_extreme2}
}
\\
\subfigure{
    \includeteximage[width=0.23\textwidth]{tikz/IHDP_dist_z_+_-.tex}
    \label{fig:IHDP_dist_extreme1}
}
\subfigure{
    \includeteximage[width=0.23\textwidth]{tikz/IHDP_dist_z_-_+.tex}
    \label{fig:IHDP_dist_extreme2}
}
\subfigure{
    \includeteximage[width=0.23\textwidth]{tikz/MNIST_dist_z_+_-.tex}
    \label{fig:MNIST_dist_extreme2}
}
\subfigure{
    \includeteximage[width=0.23\textwidth]{tikz/MNIST_dist_z_-_+.tex}
    \label{fig:MNIST_dist_extreme2}
}

\caption{Total cumulative regrets and distribution comparison for IHDP and MNIST dataset in extreme settings.\capt{The first two columns are results for the IHDP dataset, and the last two are for the MNIST dataset. The causal agent shows a constant lower regret compared to baselines, even if there is no overlapping context between the two domains.}}
\label{fig:extremes}
\end{figure}
Motivated by the performance of the causal agent, we further test it in more extreme settings. We truncated the Gaussian distribution of context $Z$ at $z=0$ to obtain two datasets, with one including all samples with context $z<0$ and the other one including all samples with context $z\geq 0$. We refer to the former as the negative truncated Gaussian data and the latter as the positive truncated Gaussian data. For the first extreme setting, we use positive truncated Gaussian data as the source domain dataset and negative truncated Gaussian as the target domain dataset. For the second extreme setting, we swap the datasets in the previous setting to create datasets for the source domain and the target domain. In these two cases, there is no overlapping between the source domain and the target domain, and the optimal action is always the opposite in the two domains. We would expect that any knowledge transfer from the source domain is hard to be helpful for the target domain performance. We show the results for both IHDP and MNIST datasets in Fig.~\ref{fig:extremes}. The first two columns are the results for the IHDP dataset, with the first column under the first extreme setting and the second column under the second extreme setting. The last two columns are results for the MNIST dataset. 

All these results indeed show the negative transfer for VAE (prior) and the causal agent. However, we notice that it takes around 20 gradient steps for the causal agent to catch up with the performance of the VAE agent, while it takes around 40 steps for the VAE (prior) agent to catch up. Moreover, the causal agent almost shows a constant lower regret than the VAE agent after 20 gradient steps. As for the learned posterior, we did not observe a large shift from the prior to the true distribution of $Z$ in the target domain, given that there is no overlapping between $\P_Z$ and $\P_Z^{\ast}$. Nevertheless, Fig.~\ref{fig:IHDP_dist_extreme1} and Fig.\ref{fig:IHDP_dist_extreme2} shows a slight shift to $P_Z^{\ast}$ in IHDP dataset. 
\end{document}

%% file: macros.tex
\renewcommand{\P}{P}
\newcommand{\PS}{P^{\ast}}
\newcommand{\ie}{i\/.\/e\/.,\/~}

\newcommand{\indep}{\perp\!\!\!\perp}
\newcommand{\capt}[1]{\mdseries{\emph{#1}}}
\newtheorem*{proposition*}{Proposition}
\newtheorem*{corollary*}{Corollary}
\DeclareMathOperator*{\argmax}{arg\,max}